\documentclass{article}

\usepackage[final]{drlw_neurips_2019}

\usepackage[utf8]{inputenc} 
\usepackage[T1]{fontenc}    
\usepackage{hyperref}       
\usepackage{url}            
\usepackage{booktabs}       
\usepackage{amsfonts}       
\usepackage{nicefrac}       
\usepackage{microtype}      
\usepackage{graphicx}
\usepackage{subfig}
\usepackage{xcolor}
\usepackage{amsmath}
\usepackage{amssymb}
\usepackage{amsthm}
\usepackage{bm}
\usepackage{mathtools}
\usepackage[makeroom]{cancel}

\DeclareMathOperator*{\argmax}{argmax}

\newtheorem{theorem}{Theorem} 

\title{Theory of Mind for \\ Deep Reinforcement Learning in Hanabi}

\author{%
  Andrew Fuchs\\
  Naval Information Warfare Center - Pacific\\
  \texttt{afuchs@spawar.navy.mil} \\
 \And
 Michael Walton\\
  Naval Information Warfare Center - Pacific\\
  \texttt{mwalton@spawar.navy.mil} \\
 \And
  Theresa Chadwick\\
  Naval Information Warfare Center - Pacific\\
  \texttt{theresa.chadwick@spawar.navy.mil} \\
 \And
  Doug Lange\\
  Naval Information Warfare Center - Pacific\\
  \texttt{dlange@spawar.navy.mil} \\
}

\begin{document}

\maketitle

\begin{abstract}
The partially observable card game Hanabi has recently been proposed as a new AI challenge problem due to its dependence on implicit communication conventions and apparent necessity of theory of mind reasoning for efficient play. In this work, we propose a mechanism for imbuing Reinforcement Learning agents with a theory of mind to discover efficient cooperative strategies in Hanabi. The primary contributions of this work are threefold: First, a formal definition of a computationally tractable mechanism for computing hand probabilities in Hanabi. Second, an extension to conventional Deep Reinforcement Learning that introduces reasoning over finitely nested theory of mind belief hierarchies. Finally, an intrinsic reward mechanism enabled by theory of mind that incentivizes agents to share strategically relevant private knowledge with their teammates. We demonstrate the utility of our algorithm against Rainbow, a state-of-the-art  Reinforcement Learning agent.
\end{abstract}

\section{Introduction}
In order to make decisions that effectively achieve their goals, humans must base their choices on the behavior of other agents whose actions may impact their respective outcomes. To navigate this complexity, humans utilize a \textit{theory of mind}: the ability of an agent to ascribe beliefs, preferences and intentions to other agents with which they interact in a shared environment. Such a theory of mind enables humans to reason recursively over nested belief states: \textit{``I believe you believe I believe,"} and make inferences about the intent or behavior of others based on this belief. For instance, when told the story, \textit{``A man walks into a room, looks around then leaves,"} human subjects often ascribe the man's behavior to his mental state and intention, that is, \textit{``The man entered the room to look for something and forgot what it was"} \cite{9780262023849}. The ability to comprehend our world and one another in terms of Theory of Mind (ToM) is fundamental to our everyday experience; however, relatively few works in reinforcement learning (RL) have successfully combined such capabilities into multi-agent decision making.

\section{Background}
\subsection{Theory of Mind}
In \cite{Byom2013}, the authors identify two relevant tools humans use in order to infer each others' mental states: interpreting the motivation governing the actions of others (intent) and inferences over shared context (beliefs). These two aspects are particularly salient to Hanabi; although partial shared context is provided by the fully observable cards, crucial context exists within the minds of the players. For instance, if a particular player would like to estimate the likelihood that another player plays their Red 2 card, one need not only guess their decision making process, but their belief state as well. This raises the question: How can the players efficiently reach agreement on conventions of play and grounded context?

As proposed in \cite{1902.00506}, we argue that ToM reasoning will be crucial for Hanabi and other related multi-agent tasks, particularly in ad-hoc settings which necessitate few-shot inference of other agents beliefs, decision making process, and hinting strategy. We begin to tackle this challenging domain by proposing a method to address the representation of nested belief in deep reinforcement learning and expand on this approach by deriving a mechanism for achieving coherent hinting strategies.

\subsection{Hanabi}
There are several aspects of the Hanabi game which make it worthy of investigation. First, it is an imperfect information game with players that have asymmetric knowledge about the environment state. Further, Hanabi is a cooperative game which requires coordination, credit assignment between agents, and learning grounded communication conventions for sharing information.

In Hanabi, agents communicate via direct information by sharing hints as well as ``play", ``discard", and ``draw" actions they may choose to take. Hints in Hanabi are restricted to only true hints and require that all cards matching the given hint be noted when the hint is given. Additionally, the set of valid hints restricts communication in such a way that complete information about a card cannot be achieved through a single hint, unless the game state provides additional context that forces the hint to explicitly define the card (e.g. only one card of a particular color remains, so a hint of that color would necessarily define the card).

We limited the scope of this paper to the self-play Sample Limited posing of the problem. Under the Sample Limited regime, agents' total experience (interactions with the environment) is limited to 100 million timesteps.

\subsection{Related Work}
\subsubsection{Rule-based Methods}
\cite{1902.00506} note several heuristic agents for playing Hanabi [HatBot \cite{cox2015make}, SmartBot \cite{smartBot}, WTFWThat \cite{WTFWThat}], which, though performant, are reliant on hand crafted rules and communication conventions.

\subsubsection{Reinforcement Learning Methods}
Closest and most relevant to our method is \textit{Policy Belief Iteration} (Pb-It) \cite{Tian2018LearningMI}, which proposes a method for learning implicit communication protocols between agents in the card game bridge. More specifically, the agents need to infer information about their hands to their partner via bidding actions. Pb-It learns a belief update function which predicts the other player's private information given their action. Their method also includes an auxiliary reward that incetivizes agents to communicate with bids that provide the largest reduction in uncertainty for their partner's belief. However, the belief update and policy functions in Pb-It must have access to private information of other players during training, which compromises the practicality of the method. Our algorithm differs crucially in that we are able to exploit a similar reward incentive without violating the rules of the game during training by explicitly sharing private information. Our algorithm uses approximations to other agents' belief distributions as a proxy for this ground truth information.

\textit{Actor-Critic-Hanabi-Agent} (ACHA) \cite{mnih2016asynchronous} is an actor-critic model parameterized with deep neural networks. ACHA uses the Importance Weighted Actor-Learner to avoid stale gradients and population-based training for hyperparameter optimization.

\textit{Rainbow} \cite{hessel2018rainbow} combines several state of the art DQN techniques, such as Double DQN, Noisy Networks, Prioritized experience replay, and Distributional RL. For our experiments, we train an implementation of Rainbow provided by \cite{1902.00506} as a comparative baseline.

\textit{Bayesian Action Decoder} (BAD) \cite{foerster2018bayesian} uses a Bayesian belief update conditioned on the acting agent's policy. In BAD, all agents utilize a public belief which incorporates all card-related common knowledge.

\subsection{Problem Statement}
Partially observable cooperative multi-agent decision making problems with stochastic state transitions are commonly modeled as Decentralized POMDPs (Dec-POMDPs). A common alternative which more naturally expresses ToM concepts of belief, preference, and desire is the Interactive POMDP (I-POMDP) framework \cite{Doshi2004AFF}. In this work, we propose a simplified ToM framework which embeds k-nested beliefs present in I-POMDPs in the belief state of a Dec-POMDP.

A Dec-POMDP is defined by $\langle S, \{A_i\}, T, R, \{\Omega_i\}, O, \gamma \rangle$. Where $S$ defines the global state, $\{A_i\}$ denotes a set of agent action spaces for each agent $i$. The state transition function $T$ defines the probability of next state observations given a current state and agent action, $T: S \times A \rightarrow \Delta(S)$ where $A = \times_i A_i$ and $\Delta(S)$ defines a set of probability distributions over $S$. $\Omega_i$ is the set of possible observations for agent $i$ with $O: S \times A \rightarrow \Delta(\Omega)$ where $\Omega = \times_i\Omega_i$. The reward function, $R: S \times A \rightarrow \mathbb{R}$, defines the global objective the team of agents seek to maximize. Lastly, $\gamma \in [0, 1)$ is a discount factor, intuitively encoding how an agent should trade off between achieving near-term versus long term rewards. Given an Dec-POMDP and a set of possible joint policies $\pi = \bigcup_i \pi_i$, our objective is to obtain an optimal joint policy which maximizes expected cumulative reward, such that $\pi^*=max_{\pi} \mathbb{E} [ \Sigma_t \gamma^t R(s_t, a_t) | \pi ]$.

Absent from the Dec-POMDP literature is a formal statement of agents' beliefs about each other's mental states. An alternative modeling framework which augments the POMDP with nested reasoning over beliefs and agents' decision making process may be found in the Interactive POMDP (I-POMDP) [e.g. \cite{Doshi2004AFF} and  \cite{Karkus2017QMDPNetDL}]. The I-POMDP is Dec-POMDP; however, it introduces the notion of a set of \textit{interacting states} $IS_i = S \times M_j$, where for each agent $i$ there is defined a set of possible models $M_j$ for other agents $j$ in the system. A rigorous treatment of I-POMDPs and solution methods for planning problems may be found in \cite{Doshi2004AFF}. Due to the recursive formulation of I-POMDPs, agents' beliefs and intentional models may be infinitely nested; making belief updates and optimality proofs intractable. Alternatively, finitely nested I-POMDPs introduce a parameter $k$ called the \textit{strategy level} which limits the depth of the recursion. For notational compactness and without loss of generality, we consider a two-player setting with players $i$ and $j$. Agent $i$'s 0-th level belief $b_{i}^0$ are probability distributions over the physical state $S$. Its 0-th level type $\Theta_{i,0}$ is a tuple containing its 0-th level belief and agent models $M_{j,0}$.\footnote{0-th level types are therefore POMDPs, as the other agents' actions are folded into  $T$, $O$ and $R$ as noise (as in naive single-agent RL applied to multi-agent settings)} An agent's first level beliefs are probability distributions over the physical state \textbf{and} the 0-th level models of the other agent. Its first level model consists of types up to level 1, second level beliefs are defined in terms of first level models, and so on and so forth.

We propose a simple extension to Dec-POMDPs to leverage ToM reasoning motivated by the I-POMDP depth finitely nested belief model. For each agent $i$, we define a ToM belief set $B^i = S_0 \times S_1 \times ... \times S_k$ which defines the set of possible worlds our agent may exist in; including at each depth $k$ other agents' perspectives defined over the corresponding state space. Therefore, a belief at depth $k$ is given by  $b_k^{ij...} = \Delta(S_k)$, such that $|ij...| = k+1$. For clarity, we denote the relationship between agents with a superscript on the belief indicating the direction of the relationship. For even $k$, the relationship is self-reflective, that is, $i$ models their own belief through recursion on themselves and $j$, for odd $k$, the relationship models other players. For instance, $b^{iji}_2$ is \textit{"I believe you believe I believe"}. We model $i$'s level $k$ utility function $Q_k^i$ associated with a particular belief state as $Q_k^i: b_0 ... \times b_k \rightarrow \mathbb{R}^{|A_i|}$. Hence, we are able to obtain a tractable ToM architecture similar to I-POMDPs though simple enough for model-free value function approximation.

\section{Methods}
We propose the use of nested theory of mind (ToM) belief levels to generate a mechanism for performing and incentivizing communication for MARL agents in the Hanabi environment. These belief levels not only allow agents to estimate the portion of the state that is not fully observable, but also supply a method of simulating the same belief estimate for other agents. The ability to simulate a belief level for other agents allows each to recursively reason about the other agents' ToM beliefs. This recursive reasoning can be taken to arbitrary depth. Our method includes two levels of belief; we argue this is sufficient for efficient gameplay in Hanabi and analogous to human reasoning on this task. We demonstrate that ToM-based agents will more effectively learn the mechanics of the game to play through information-maximising hint conventions and belief update strategies.

\subsection{Nested Beliefs in Hanabi ToM}
Given the nature of the Hanabi game, ToM agents must first estimate a belief regarding their own cards, which we denote $b_0^i$ for player $i$. The $b_0^i$ distribution provides the probability of a hand of cards for player $i$ given the observed cards for other players, hints received, cards in the discard pile, and cards on the firework piles. The information provided to the agents allows for a reduction in the sample space by eliminating cards which violate the hints observed as well as any cards that have been exhausted through gameplay. With this reduced sample space, we can calculate a hand probability based on sampling without replacement, noted in equation \ref{belief_eq09}.

The following definitions outline the characteristics of our belief levels. Let $F = \{R, G, B, W, Y\}\label{belief_eq01}$ and $V = \{1, 2, 3, 4, 5\}\label{belief_eq02}$ denote the sets of possible colors and ranks respectively for cards in the game. Then the possible knowledge for any given card is defined by $H_F = F \cup\{\varnothing\}\label{belief_eq03}$ and $H_V = V \cup\{\varnothing\}\label{belief_eq04}$, where $\varnothing$ denotes no hint given. Therefore, for each player $i\in [1,N]$, their cards and hand knowledge are defined by $C^i = (c_1,\dots c_\eta), c_k\in F\times V, \eta\in [4,5]\label{belief_eq05}$ and $H^i = (h_1,\dots,h_\eta), h_k\in H^F\times H^V, \eta\in [4,5]\label{belief_eq06}$, respectively. Given the observed hands and hints regarding their own cards, players can generate their knowledge set and the unique set of hints:
\begin{align}
    K^i &= \{(c_k, h_k): k\in [0, \eta], c_k\in C^i, h_k\in H^i\}\label{belief_eq07}\\
    H^* &= \{h:\exists_{k\in [1,\eta ]}\ni h=h_k, \forall h_k\}\label{belief_eq08}
\end{align}
The knowledge set and hint set provide the cases required to calculate a hand's probability given the current game state. Given $K^i$ and $H^*$, players estimate the probabilities for belief level 0 and belief level 1 using:
\begin{align}
    b_0^i&\sim P(C^i|H^i,\eta) := \frac{\prod_{(c,h)\in K^i}(n_\nu(c|h))_{\delta^i_c(h)}}{\prod_{h\in H^*}({|\nu(h)|})_{\lambda^i(h)}}\label{belief_eq09}\\
    \hat{b}_0^i&:=\frac{\prod_{(c,h)\in K^i}n_\nu(c|h)}{\prod_{h\in H^*}{|\nu(h)|}}\label{belief_eq10}\\
    \hat{b}_1^{ij} &:= \frac{1}{Z}\sum_{C^i\sim\hat{b}_0^i}P(C^i|H^i, \nu)\hat{b}_0^j\label{belief_eq11}\\
    \hat{b}_0^j &:= P(C^j|H^j,C^i,\hat{\nu}^j)\label{belief_eq12}
\end{align}
where:
\begin{align*}
    \nu &= \textrm{multiset} (F\times V, n_\nu(c)) \\
    \delta^i_c(h) &= \textrm{multiplicity of } (c, h) \textrm{ for hand } C^i\\
    \lambda^i(h) &= \textrm{multiplicity of } h \textrm{ for hints } H^i\\
    \nu(h) &= \textrm{sub-multiset of } \nu \textrm{ given hint } h\\
    n_\nu(c|h) &= \textrm{multiplicity of card } c \textrm{ in }\nu(h)\\
    n_{C^i}(c) &= \textrm{multiplicity of card } c \textrm{ in hand } C^i\\
    |\nu(h)| &= \sum_{c\in Supp(\nu(h))}n_{\nu}(c|h)\\
    \hat{\nu}^j &= \{F\times V, \hat{n}_\nu\}\\
    \hat{n}_\nu(c) &= n_\nu(c) + n_{C^i}(c)  - n_{\hat{C}^i}(c),\\
    & (\hat{C}^i \textrm{ is sampled, } C^i \textrm{ is observed by } j)
\end{align*}

For clarity, a working example of the above (equations 3 - 6) is provided in Appendix A.2. For tractability and to avoid enumerating all possible hands, we estimate each player's belief $b_0^i$ using the estimated belief defined by equation \ref{belief_eq10}, which allows us to represent the card probabilities as independent samples with replacement when generating sample hands.

We defined two methods for sampling hands for player $i$ when calculating $b_1^{ij}$. In the first case, we generate a maximum a posteriori hand $C_{MAP} = \argmax_{C^i} P(C^i|H^i)$, which is then used to augment the observed player hands from player $j$'s perspective. The augmented player hand is used to generate $\hat{\nu}^j$, which is then used in the calculation of $b_0^j$. The resulting $b_0^j$ is then used as $b_1^{ij}$. This heuristic approximation is used to mimic a particular interpretation of human play: humans may reason about one another's beliefs with respect to their most likely hand, rather than integrate over all possible hands weighted by $b_0^i$.

In the alternative case where we do not use the $C_{MAP}$ method, hands are sampled proportional to equation \ref{belief_eq10} and then selected given $b_0^i(C^i) > 0$ until the number of samples matches the requested number of hands. The resulting sampled hands are then used as input for equations \ref{belief_eq11} and \ref{belief_eq12}, which estimate the level-0 belief for other agents based on player $i$'s estimated hand, and generate a weighted average based on probability of $C^i\sim b_0^i$. The resulting weighted average is normalized in equation 5 and used to represent player $i$'s estimate of the 0-level belief for agent $j$.

\subsection{Incentivising Efficient Hints}

Given that each agent has access to belief level $b_1^{ij}$ by equations \ref{belief_eq10} and \ref{belief_eq12}, it is natural to consider how this approximation may be further leveraged to encourage Hanabi agents to make maximally informative hint actions. Since each agent $C^j$ is in the observation of every other agent $i$, each agent may independently compute $D(C^j ~|| ~b_1^{ij})$, where $D$ is some divergence measure between the true hand and $i$'s approximation to $j$'s belief. Partially motivated by \cite{Tian2018LearningMI} we encourage our agents to minimize this quantity by providing the intrinsic reward $r^i_{c,t}$ given by:
\begin{align}
        r_{c,t}^i &= \max_{j\ne i} \sum_{k=1}^\eta\bigl[{W_p(\psi(C_t^j[k]),\hat{b}_{1;t}^{i j})}  - W_p(\psi(C_{t+1}^j[k]),\hat{b}_{1;t+1}^{i j})\bigr]
\end{align}

\noindent
where $W_p(\psi(C_t^j[k]),\hat{b}_{1;t}^{ij})$ is the Wasserstein metric and $\psi(C_t^j[k])$ is the one-hot encoding of the cards in player $i$'s hand. In our experiments we use $p=2$ and weight the intrinsic communication reward by a hyperparameter $\beta$, which roughly trades off between sharing information via hint actions and engaging in more aggressive play and discard strategies. We hypothesize that using this reward term usefully biases the agents' value function towards more conservative play and helps infer the effectiveness of hint actions. This is accomplished by determining how incorrect our estimated belief for other players are by comparing to their ground truth hands. As such, we can directly measure how much a given hints shifts the distribution to match the ground truth and determine whether there is a belief significantly different from the truth to warrant performing a hint action rather than a play action.

\subsection{ToM Rainbow Agent}
Own-hand beliefs $b_0^i$ were represented as flattened vectors $b_0^i \in \mathbf{R}^{\eta \times |F\times V|}$ which replace the partial card knowledge encoding component of the Hanabi environment's default observation vector. In experiments which reason over first-order beliefs, both belief levels $\{b_0^i,b_1^{ij}\}$ were concatenated. The resultant modified observation vector is roughly the same dimension as the default in Hanabi. Our models do not use or require observation history stacking, since $b^0_i$ is a sufficient statistic for any other agent $j$'s private knowledge of $i$'s hand given the most recent hint and $i$'s observation.

\begin{figure}[t!]
    \centering
    \includegraphics[scale=.22]{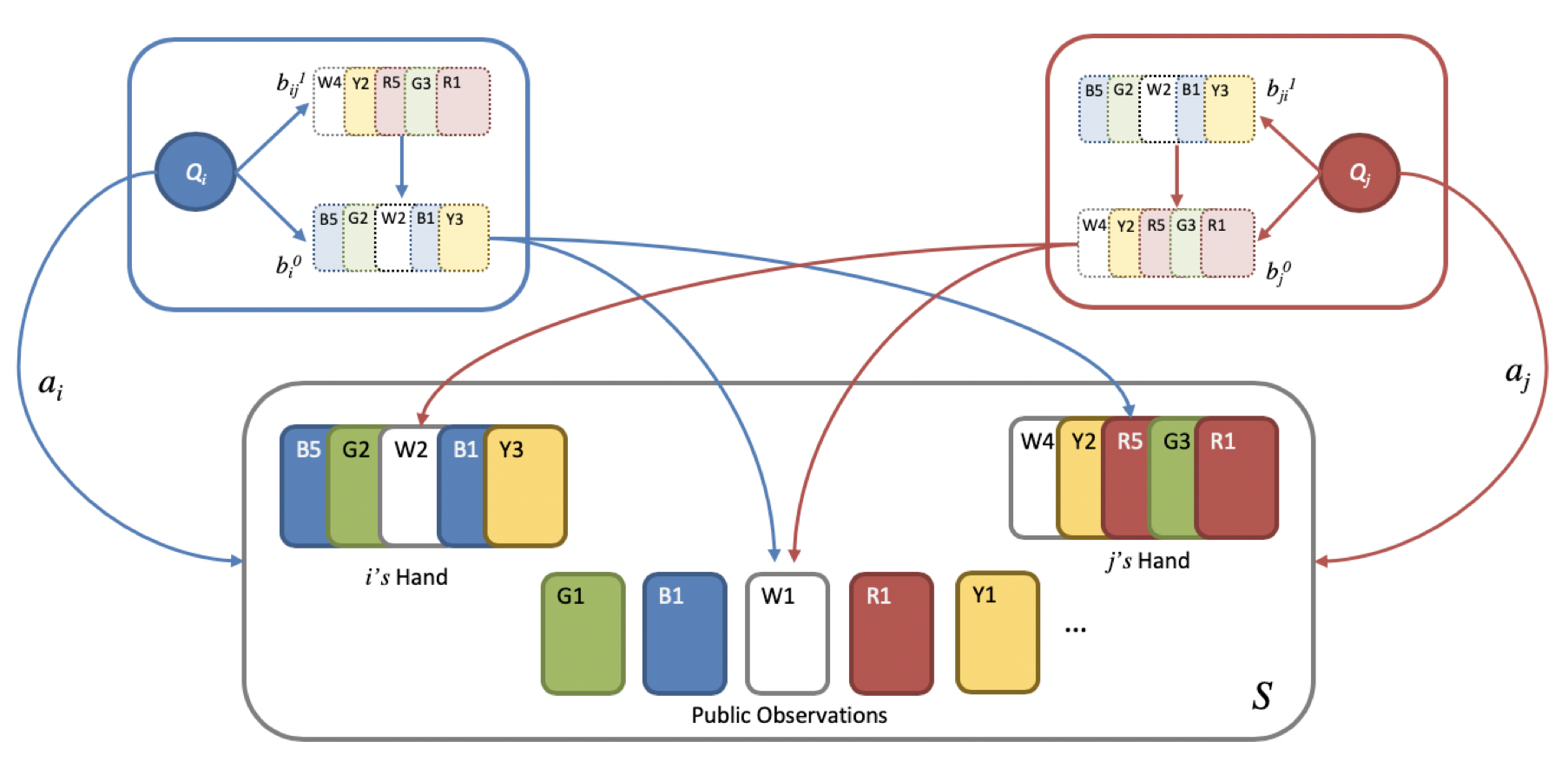}
\caption{ToM architecture for Hanabi. Fully colored cards denote the state of the game board and pastel colored hands indicate agents' belief over unobserved parts of the game state and one anther's beliefs, given their observations.}
\end{figure}

Computing $b_1^{ij}$ requires marginalizing over the agent's own-hand belief $b^0_i$. We have proposed two mechanisms for mitigating this issue, the latter of which we explore empirically. The first is to obtain a Monte-Carlo approximation of $b^j_0$ by iteratively sampling from agent $i$'s own-hand belief, computing $b^0_j$ given each sample, and weighting the associated probabilities by the probability of the sampled hand under agent $i$'s current own-hand belief from equation \ref{belief_eq11}.

We observed that the empirical runtime cost of the MC variant of our algorithm scales roughly linearly in the number of samples drawn from $b^{ij}_1$. Small sample sizes produced high variance estimates of $b^{ij}_1$ and empirically lead to low scores. In order to yield a more computationally tractable algorithm for estimating $b^{ij}_1$, we instead compute this approximation given the current maximum a posteriori hand under agent $i$'s own-hand belief $C^j_{MAP}$.

All agents including the baseline were trained using the value-based Rainbow algorithm and shared all associated hyperparameters. Unless otherwise noted, we train with $\gamma=0.99$, $V_{max}=25$, update horizon of one and 51 atoms to approximate the value function distribution. We train our models using using RMSProp \cite{Tieleman2012} with a learning rate of $0.0025$, decay of $0.95$, 0.0 momentum, and $\epsilon=10^6$.

\subsection{Experiments}
We first assess total scores and sample efficiency of our method against a baseline Rainbow agent. As described, these experiments share hyperparameters with the code associated with the Hanabi Rainbow baseline in \cite{1902.00506}. All models were trained in the two player sample limited version of Hanabi for 20k iterations of 500 episodes each. We present final evaluation of all agent scores at $10^8$ total timesteps (individual interactions with the environment) to conform with the sample-limited regime guidelines.

Intuitively, our intrinsic reward mechanism should increase the frequency of hint actions that are informative under the current policy and state of the game board. We explore whether this is in fact the case by first assessing the sample efficiency and final game score achieved by ToM depth-1 agents trained with our hint reward weighted by a range of $\beta$ values, where $\beta=(0,2,7,25)$.

\section{Results}
Our full method shows improved sample efficiency during the early stages of training as well as better evaluation scores over the baseline. Our model's performance in self-play after being trained also demonstrates slightly reduced variance against the baseline. In ablation experiments, we found that agents that implemented belief level 1 were at least as performant as belief level 0 (in terms of median scores) and outperform belief level 0 in the three and four player cases. Both methods show a larger margin of improvement over the baseline Rainbow agent in settings with a greater number of players; with an improvement of two points in the five player setting. It is interesting to note that simply providing the nested belief representation to our agents did not improve their utility above the baseline model. However, in combination with the proposed intrinsic reward, agents score above our baseline model and achieve scores exceeding those of the best sample limited agent reported in \cite{1902.00506}.

\begin{figure}[h!]
    \begin{center}
    \includegraphics[width=.4\textwidth]{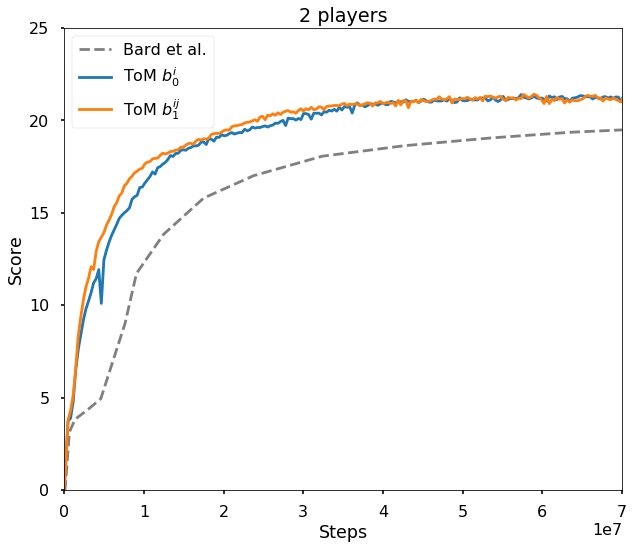}
    \includegraphics[width=.4\textwidth]{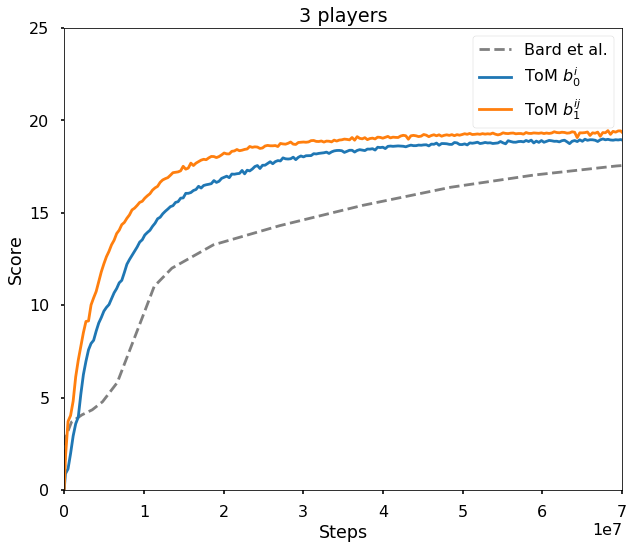}
    \includegraphics[width=.4\textwidth]{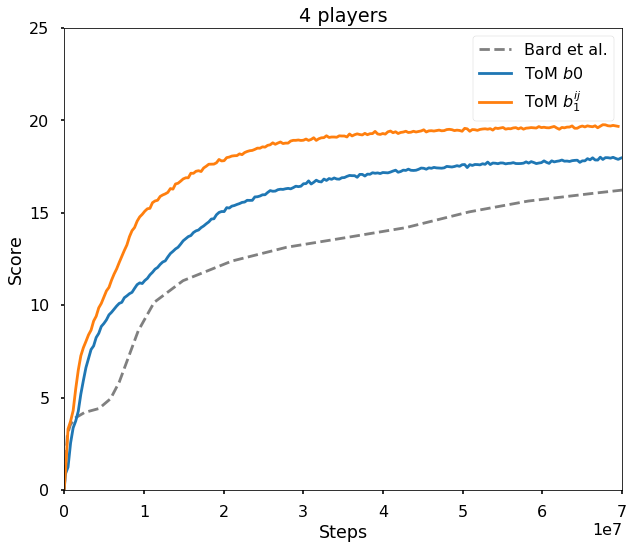}
    \includegraphics[width=.4\textwidth]{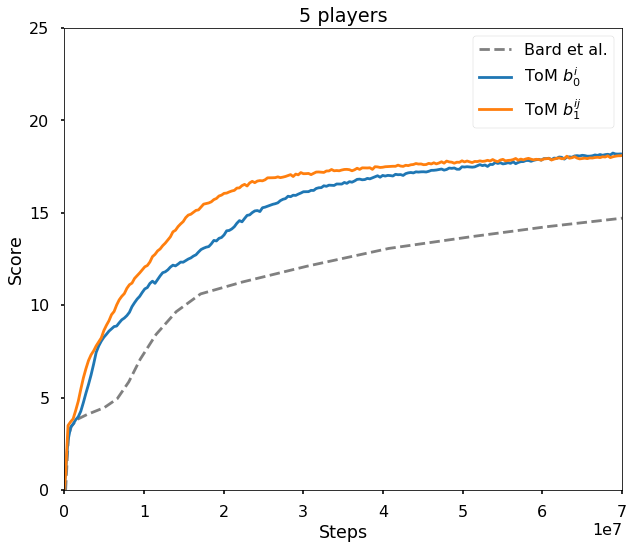}
    \end{center}
    \caption{Average Score Per iteration, 2 and 3 player games in sample-limited self-play}
    \label{fig:my_label}
\end{figure}

\begin{table}[]
\centering
\begin{tabular}{|l|cccc|}
\hline
Agent          & 2P                                                                & 3P                                                                & 4P                                                                & 5P                                                                \\ \hline
Rainbow        & \begin{tabular}[c]{@{}c@{}}21 \\ 20.64 (.22)\end{tabular}         & \begin{tabular}[c]{@{}c@{}}19\\ 18.71 (.20)\end{tabular}          & \begin{tabular}[c]{@{}c@{}}18\\ 18.0 (.17)\end{tabular}           & \begin{tabular}[c]{@{}c@{}}17\\ 15.26 (.18)\end{tabular}          \\ \hline
ToM $b_0^i$    & \textbf{\begin{tabular}[c]{@{}c@{}}22\\ 21.55 (.07)\end{tabular}} & \begin{tabular}[c]{@{}c@{}}19\\ 19.78 (.07)\end{tabular}          & \begin{tabular}[c]{@{}c@{}}18\\ 17.40 (.09)\end{tabular}          & \textbf{\begin{tabular}[c]{@{}c@{}}19\\ 18.87 (.07)\end{tabular}} \\ \hline
ToM $b_1^{ij}$ & \textbf{\begin{tabular}[c]{@{}c@{}}22\\ 21.43 (.08)\end{tabular}} & \textbf{\begin{tabular}[c]{@{}c@{}}20\\ 19.76 (.08)\end{tabular}} & \textbf{\begin{tabular}[c]{@{}c@{}}19\\ 19.13 (.09)\end{tabular}} & \textbf{\begin{tabular}[c]{@{}c@{}}19\\ 18.49 (.06)\end{tabular}} \\ \hline
\end{tabular}%
\caption{Results for the three agents: Rainbow, ToM belief level 0 $b_0^i$ and belief level 1 $b_1^{ij}$. Shown in each row are the median scores of trained agents over 1000 episodes of self-play followed by mean scores and (standard error of the mean). Rainbow scores are the best agent results from \cite{1902.00506} ToM agents are the results of our proposed agents}
\label{tab:my-table}
\end{table}

\begin{figure}[h!]
    \centering
    \includegraphics[width=0.45\textwidth]{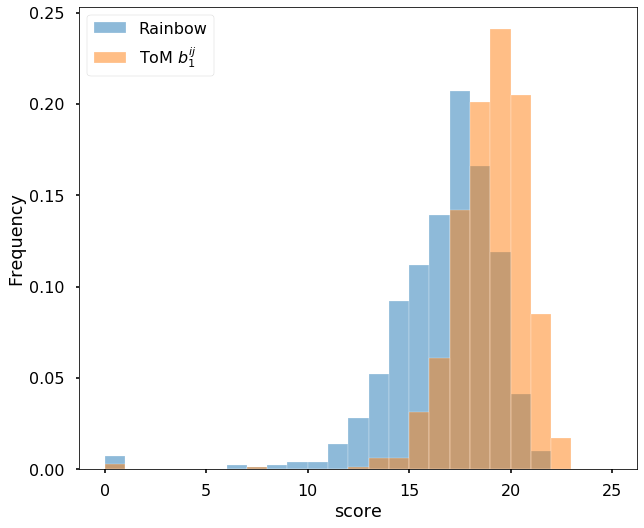}
    \caption{Score distributions from best Rainbow vs ToM agents. Score distributions were generated by running each agent for 1000 episodes in self-play after training}
\end{figure}

We next report the impact of our intrinsic communication reward for various values of $\beta$ which trades off between the weight placed on the environmental reward, which is always the final score of the game when it terminates, versus the intrinsic reward. We observe that $\beta=2$ achieves higher scores than other alternatives. This performance increase is fairly robust to a wide range of choices for this hyperparameter. However, in Hanabi, we observe performance falls off sharply for $\beta > 10$. We speculate that in such situations, agents trained with too large a $\beta$ will exhaust their communication tokens attempting to fully reveal one another's hands as quickly as possible while ignoring the task of building the required firework decks and winning the game.

\begin{figure}[h!]
    \centering
    \includegraphics[width=0.46\textwidth]{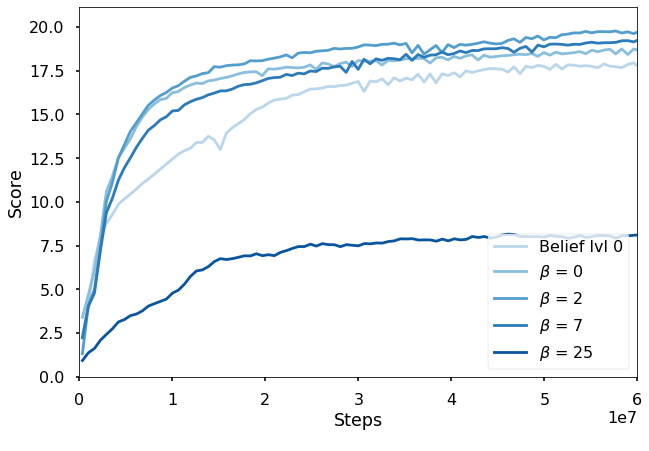}
    \includegraphics[width=0.44\textwidth]{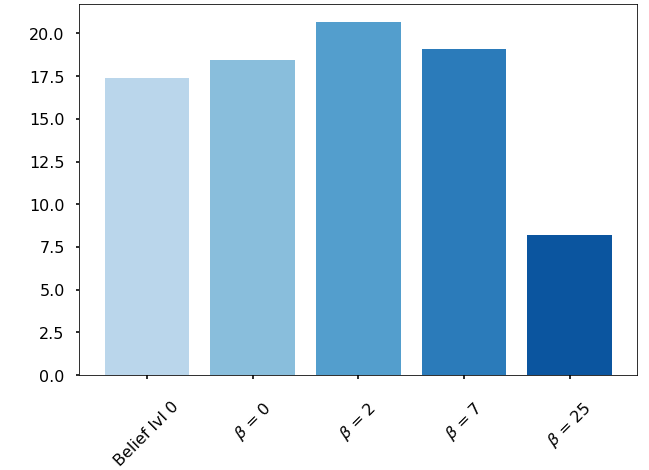}
\caption{Effect of intrinsic communication reward on expected return. (left) Training curves for different choices of $\beta$. (right) Resulting scores after training for the same values of $\beta$}
\end{figure}

\section{Discussion}

We have proposed an approach to achieving theory of mind reasoning over nested beliefs using reinforcement learning for Hanabi agents. We have defined the necessary hint-conditional posterior distribution over player hands and proposed two efficient algorithms for constructing approximations to other players' belief about their own hands. With these tools we are able to incentivize efficient hint actions using a unique intrinsic reward term defined over the other agent's belief.

Our approach shows promise for further investigation in both Hanabi and other ToM reasoning tasks. In future work, it will be interesting to empirically assess the scalability of our algorithm to deeper belief nesting. We  conjecture that the intrinsic reward mechanism elicits changes in the action distribution that cannot be clearly intuited from measuring extrinsic reward. Also, we will explore methods for characterizing the strategies learned by our method to more efficiently identify the grounded hint conventions RL agents may be learning from the inclusion of nested ToM beliefs.

\bibliographystyle{plainnat}
\bibliography{neurips_2019}

\newpage
\begin{table*}
\centering
\section*{Appendix}
\subsection*{A.1 Proof Equation \ref{belief_eq09} Sums to 1}
\setcounter{section}{1}
\begin{theorem}
$\displaystyle \sum\limits_{C^i = (c_1, c_2,...c_\eta) \in \bigtimes \limits_{\eta, h \in H^i}  \nu(h)} P(C^i | H^i, \eta) = 1$
\end{theorem}
\begin{proof}[Proof]
Assume there are $s$ unique hints, so $|H^*| = s$ and $r$ unique hint-card combinations, so $|K^i| = r$. 
  
\hspace{3mm} $\sum\limits_{C^i = (c_1, c_2,...c_\eta) \in \bigtimes \limits_{\eta, h \in H^i}  \nu(h)} P(C^i | H^i, \eta)$

$ = \sum\limits_{C^i \in \bigtimes \limits_{\eta, h \in H^i}  \nu(h)} \frac{\prod \limits_{(c,h)\in K^i}(n_\nu (c|h))_{\delta^i_c(h)}}{ \prod\limits_{h\in H^*}({|\nu(h)|})_{\lambda^i(h)}}$

$= \frac{{ \sum\limits_{{C}^i \in \bigtimes \limits_{\eta, h \in H^i}  \nu(h)} }\hspace{1mm}  \prod \limits_{(c,h) \in {K}^i} (n_\nu (c|h))_ {\delta^i_c(h)}}{\prod \limits_{h \in H^*} (|\nu(h)|)_{\lambda^i(h)}} $

$ = \frac{\sum\limits_{{C}^i \in \bigtimes \limits_{\eta, h \in H^i}  \nu(h)} \hspace{1mm}  \prod \limits_{(c,h) \in {K}^i} (n_\nu (c|h))_{\delta_c^i(h)}}{\prod \limits_{h \in H^*} \left( \sum\limits_{c\in Supp(\nu(h))} n_\nu (c|h) \right)_{\lambda^i(h)}}  $

$= \frac{\sum\limits_{{C}^i \in \bigtimes \limits_{\eta, h \in H^i}  \nu(h)} \left[ 
 (n_\nu (c_1|h_1))_{\delta_c^i(h_1)} \cdots (n_\nu (c_k|h_r))_{\delta_c^i(h_r)}
  \right] }{ 
 \left(\sum\limits_{c\in Supp(\nu(h_1))} n_\nu (c|h_1)\right)_{\lambda^i(h_1)} 
\cdots 
\left(\sum\limits_{c\in Supp(\nu(h_s))} n_\nu (c|h_s)\right)_{\lambda^i(h_s)} 
} 
$\\
\vspace{2.5mm}
Factor out the hint $h_r$ from the sum over all possible card hands on the numerator, which consequently requires factoring out all $c \in C^i$ associated with that hint. Those card-hint combinations can be expressed by the set $K^* \subseteq K^i \ni h_k = h_r$. Without loss of generality, assume $h_r = h_s$, thus $\sum\limits_c \delta^i_c(h_r) = \lambda^i(h_r) = \lambda^i(h_s)$.

$ = \frac{ 
	\sum\limits_{{C}^i \in \bigtimes \limits_{\eta - \lambda^i(h_r), h \in H^i}  \nu(h)} 		
	\left[ 
		 (n_\nu (c_1|h_1))_{\delta_c^i(h_1)} 
           \cdots 
		 (n_\nu (c_{r-\lambda^i(h_s)}|h_{k-\lambda^i(h_s)}))_{\delta_c^i(h_{r-\lambda^i(h_s)})}
 	 \right] 
 	\cancel{\left(\sum\limits_{c\in Supp(\nu(h_r))} n_\nu (c|h_r)\right)_{\sum\limits_c \delta_c^i(h_r)}}
      }{ 
	\left[
		 \left(\sum\limits_{c\in Supp(\nu(h_1))} n_\nu (c|h_1)\right)_{\lambda^i(h_1)} 
	\hspace{1mm}  \cdots \hspace{1mm}
		\cancel{\left(\sum\limits_{c\in Supp(\nu(h_s))} n_\nu (c|h_s)\right)_{\lambda^i(h_s)}} 
	\right]
    } $
\\
\vspace{2.5mm}
Continue on for all unique hints.\\
\vspace{1.5mm}

\hspace{.5mm} \vdots\\
\vspace{2.5mm}

$= 1$

\noindent\\
\vspace{2.5mm}
$\therefore \sum\limits_{C^i = (c_1, c_2,...c_\eta) \in \bigtimes \limits_{\eta, h \in H^i}  \nu(h)} P(C^i | H^i, \eta) = 1$

\end{proof}
\end{table*}
\begin{table*}
\subsection*{A.2 Equations 3 - 6 Example}
\begin{enumerate}
	\item Given hints: [R, G, R, 1, 2]
	\item Valid Cards and Counts given hint i.e. $n_\nu(c|h)$ Counting of 0 for any card not listed:
	\begin{itemize}
		\item Card1: [R3: 2, R4: 2, R5: 1]
		\item Card2: [G3: 2, G4: 2, G5: 1]
		\item Card3: [R3: 2, R4: 2, R5: 1]
		\item Card4: [Y1: 3, W1: 3, B1: 3]
		\item Card5: [Y2: 2, W2: 2, B2: 2]
	\end{itemize}
	\item Calculating $b_0^i$ using samples without replacement (equation 3): \\
		$b_0^i \sim P([R3, G3, R4, W1, W2]|[R, G, R, 1, 2], 5)$ 
		\\ \hspace{35pt} = $(2/5)(2/5)((2 - 1)/(5-1))(3/9)(2/6)$
		\\ \hspace{35pt} = $1/225$
	\item Calculating $\hat{b}_0^i$ using samples with replacement (equation 4): \\
$\hat{b}_0^i = (2/5)(2/5)(2/5)(3/9)(2/6) = 8/1125$
	\item Calculating $\hat{b}_0^{ij}$ follows the same method as $b_0^i$, but first relies on modifying the counts to approximate agent $j$'s perspective. This is done by using a hand sampled by $b_0^i$ as the cards observed for player $i$'s hand from $j$'s perspective. Additionally, the cards observed in $j$'s hand from $i$'s perspective are added back to the set of unobserved cards. This gives all the necessary information for repeating equation 3 for agent $j$'s perspective.
	\item Equation 6 is equivalent to equation 4, but with the same updated counts as defined in the previous section to simulate $j$'s perspective.

\end{enumerate}

\end{table*}

\end{document}